\def\year{2019}\relax
\documentclass[letterpaper]{article} 
\usepackage{aaai19}  
\usepackage{times}  
\usepackage{helvet}  
\usepackage{courier}  
\usepackage{url}  
\usepackage{graphicx}  
\usepackage{natbib}

\usepackage{dsfont}
\usepackage{amsmath}
\usepackage{amsfonts}
\usepackage{amssymb}
\usepackage{amsthm}
\usepackage{subfigure}
\usepackage{epsfig}
\usepackage{color}
\usepackage{enumerate}
\usepackage{placeins}

\usepackage{times}
\usepackage{graphicx} 
\usepackage{subfigure} 

\usepackage{algorithm}
\usepackage{algorithmic}

\usepackage[dvipsnames]{xcolor}

\usepackage{hyperref}

\usepackage{hyperref}

\newtheorem{lemma}{Lemma}
\newtheorem{theorem}{Theorem}
\newtheorem{corollary}{Corollary}
\newtheorem{definition}{Definition}
\newtheorem{assumption}{Assumption}
\newtheorem{remark}{Remark}

\theoremstyle{definition}

\DeclareMathOperator*{\argmax}{argmax}

\begin{document}
%

\title{Non-Asymptotic Uniform Rates of Consistency for $k$-NN Regression}
\author{Heinrich Jiang\\Google Research\\Mountain View, CA\\heinrich.jiang@gmail.com
}
\maketitle
\begin{abstract}
We derive high-probability finite-sample uniform rates of consistency for $k$-NN regression that are optimal up to logarithmic factors under mild assumptions. We moreover show that $k$-NN regression adapts to an unknown lower intrinsic dimension automatically in the sup-norm. We then apply the $k$-NN regression rates to establish new results about estimating the level sets and global maxima of a function from noisy observations.
\end{abstract}

\section{Introduction}

The $k$-nearest neighbor ($k$-NN) regression algorithm is a classical approach to nonparametric regression.
The value of the functional
is taken to be the unweighted average
observation of the $k$ closest samples.
Although this procedure has been known for a long time and has a deep practical significance,
there is still surprisingly much about its convergence properties yet to be understood.

We derive finite-sample high probability uniform bounds for $k$-NN regression under a standard additive model 
$y = f(x) + \xi$ where $f$ is an unknown function, $\xi$ is sub-Gaussian white noise
and $y$ is the noisy observation. The samples $\{(x_i,y_i)\}_{i=1}^n$ are drawn i.i.d. as follows: $x_i$ is drawn according to an unknown density $p_X$, which shares the same support as $f$, and then observation $y_i$ is generated by the additive model based on $x_i$.

We then
give simple procedures to estimate the level sets and global maximas of a function given noisy observations and
apply the $k$-NN regression bounds to establish new Hausdorff recovery guarantees for these structures.
Each of these results are interesting on their own.

The bulk of the work on $k$-NN regression convergence theory is on its properties under various risk measures or asymptotic convergence.
Notions of consistency involving risk measures such as mean squared error are considerably weaker than the
sup-norm as the latter imposes a {\it uniform} guarantee
on the error $|f_k(x) - f(x)|$ where $f_k$ is the $k$-NN regression estimate of function $f$.
Existing work on studying $f_k$ under the sup-norm thus far are asymptotic.
We give the first sup-norm {\it finite-sample} result. This result matches the minimax optimal rate up to logarithmic factors.

We then discuss the setting where the data lies on a lower dimensional manifold. It is already known that $k$-NN regression
is able to automatically adapt to the intrinsic dimension under various risk measures: the rates depend only on the intrinsic dimension and independent of ambient dimension. We show that this is also the case in the sup-norm: we attain finite-sample bounds as if we were operating in the lower intrinsic dimension space without any modifications to the procedure. 

We then show the utility of our $k$-NN regression results in recovering certain structures of an arbitrary function, namely the level-sets and global maximas.
The motivation can be traced back to the rich theory of density-based clustering.
There, one is given a finite sample from a probability density $p$.
The clusters can then be modeled based on certain structures in the underlying density $p$. Such structures include
the level-sets $\{ x : p(x) \ge \lambda\}$ for some density level $\lambda$ or the local maximas of $p$.
Then to estimate these, one typically uses a plug-in approach using a density estimator $\widehat{p}$
(e.g. for level-sets, $\{ x : \widehat{p}(x) \ge \lambda\}$ and for modes, $\argmax_x \widehat{p}(x)$).
It turns out that given uniform bounds on $\widehat{p}$, we can estimate these structures with strong guarantees.

In this paper, instead of estimating these structures in a density, we estimate these structures for a general function $f$.
This is possible because of our established finite-sample sup-norm bounds for nonparametric regression.
There are however some key differences in our setting.
In the density setting, one has access to i.i.d. samples drawn from the density.
Here,  we have an i.i.d. sample $x$ drawn from some density $p_X$ not necessarily related to $f$, and then we obtain a noisy observation of the value $f(x)$. This
can be viewed as a noisy observation of the {\it feature} of $x$.
In other words, we estimate the stuctures based on the features of data, while in the density setting, there are no features
and the structures are instead based on the dense regions of the dataset.

\section{Related Works and Contributions}

\subsection{$k$-NN Regression Rates}

The consistency properties of $k$-NN regression have been studied for a long time and we highlight some of the work here.
\citet{biau2010rates} give guarantees under $L_2$ risk.
\citet{devroye1994strong} give consistency guarantees under the $L_1$ risk.
\citet{stone1977consistent} provides results under $L_p$ for $p \ge 1$.
All these notions of consistency so far are under some integrated risk, and thus are weaker than the sup-norm (i.e. $L_\infty$), which imposes a uniform
guarantee. 

A number of works such as 
\citet{mack1982weak, cheng1984strong,devroye1978uniform, lian2011convergence, kudraszow2013uniform} give strong uniform convergence rates. However,
these results are asymptotic. Our bounds explore the {\it finite-sample} consistency properties of $k$-NN regression, which we will demonstrate later can show strong results about $k$-NN based learning algorithms which were not possible with existing results.
To the best of our knowledge, this is
the first such finite-sample uniform consistency result for this procedure, which matches the minimax rate up to logarithmic factors. 

We then extend our results to the setting where the data lies on a lower dimensional manifold.
This is of practical interest because the curse of dimensionality forces nonparametric methods such as $k$-NN to require 
an exponential-in-dimension sample complexity; however as a concession, we can show that many of these methods can have sample complexity depending on the intrinsic dimension (e.g. doubling dimension, manifold dimension, covering number) and independent of the ambient dimension.
In modern data applications where the dimension can be arbitrarily high, oftentimes the number of degrees of freedom remains much lower. It thus becomes important 
to understand these methods under this setting.

\citet{kulkarni1995rates} give results for $k$-NN regression based on the covering numbers of the support of the distribution.
\citet{kpotufe2011k} shows that $k$-NN regression actually adapts to the local intrinsic dimension without any modifications to the procedure or data in the $L_2$ norm. In this paper, we show that this holds in the sup-norm as well for a global intrinsic dimension.

\subsection{Level Set Estimation}

Density level-set estimation has been extensively studied and has significant implications to density-based clustering. Some works include
\citet{tsybakov1997nonparametric,singh2009adaptive}.
It involves estimating $L_p(\lambda) := \{x : p(x) \ge \lambda\}$ given a finite i.i.d. sample $X$ from $p$, where $\lambda$ is some known density level and $p$ is the unknown density.
$L_p(\lambda)$ can be seen as the high density regions of the data and thus the connected components can be used as the core-sets in clustering.
It can be shown that given a density estimator $\widehat{p}_n$ with guarantees on $|\widehat{p}_n - p|_\infty$, 
then taking $\widehat{L}_p(\lambda) := \{ x \in X : \widehat{p}_n (x) \ge \lambda\}$, the Hausdorff distance 
between $L_p(\lambda)$ and $\widehat{L}_p(\lambda)$ can also be bounded. 

In this paper, we extend this idea to functions $f$ which are not necessarily densities given noisy observations of $f$.
We obtain similar results to those familiar in the density setting, which are made possible by our established bounds for estimating $f$.
An advantage of this approach is that it can be applied to clustering where there are features where clusters are defined 
as regions of similar feature value rather than similar density. In density-based clustering, it is typical that one does not assume access to the features and thus such procedures fail to readily take advantage of the features when performing clustering.
A similar approach was taken by \citet{willett2007minimax} by using nonparametric regression to estimate the level sets of a function; however our consistency results are instead under the Hausdorff metric.

\subsection{Global Maxima Estimation}

We next give an interesting result for estimating the global maxima of a function.
Given $n$ i.i.d. samples from some distribution on the input space and seeing a noisy observations of $f$ at the samples,
we show a guarantee on the distance between the sample point with the highest $k$-NN regression value and the (unique) point which maximizes $f$.
This gives us insight into how well a grid search or randomized search can estimate the maximum of a function.

This result can be compared to mode estimation in the density setting where the object is to find the point which maximizes the density function \cite{tsybakov1990recursive}. \citet{dasgupta2014optimal} show that given $n$ draws from a density, the sample point which maximizes the $k$-NN density estimator is close to the true maximizer of the density; moreover they give finite-sample rates. Earlier works such as \citet{romano1988weak} provide asymptotic rates.

\section{$k$-NN Regression}

Throughout the paper, we assume a function $f$ with compact support $\mathcal{X}\subseteq \mathbb{R}^D$
and that we have datapoints $(x_1,y_1),...,(x_n,y_n)$ drawn follows.
The $x_i$'s are drawn i.i.d. from density $p_X$ with support $\mathcal{X}$.
Then $y_i = f(x_i) + \xi_{x_i}$ where $\xi_{x_i}$ are i.i.d. drawn according to random variable $\xi$.

\begin{definition}
$f : \mathcal{X} \rightarrow \mathbb{R}$ where $\mathcal{X} \subseteq \mathbb{R}^D$ is compact. 
\end{definition}

The first regularity assumption ensures that the support $\mathcal{X}$ does not become arbitrarily thin anywhere. Otherwise,
it becomes impossible to estimate the function in such areas from a random sample.
\begin{assumption}[Support Regularity] \label{a1}
There exists $\gamma > 0$ and $r_0 > 0$ such that $\text{Vol}(\mathcal{X} \cap B(x, r)) \ge \gamma \cdot \text{Vol}(B(x, r))$ for all $x \in \mathcal{X}$ and $0 < r < r_0$.
\end{assumption}

The next assumption ensures that with a sufficiently large sample, we will obtain a good covering of the input space.
\begin{assumption} 
[$p_X$ bounded from below] \label{a2} $p_{X, 0} := \inf_{x \in \mathcal{X}} p_X(x) > 0$.
\end{assumption}

Finally, we have a standard sub-Gaussian white noise assumption in our additive model.
\begin{assumption} [Sub-Gaussian White noise] \label{a3}
$\xi$ satisfies
$E[\xi] = 0$ and sub-Gaussian with parameter $\sigma^2$ (i.e. $E[\exp(\lambda\xi)] \le \exp(\sigma^2\lambda^2/2)$ for all $\lambda \in \mathbb{R}$).
\end{assumption}

Then define $k$-NN regression as follows.
\begin{definition} [$k$-NN]
Let the $k$-NN radius of $x \in \mathcal{X}$ be $r_k(x) := \inf \{ r : |B(x, r) \cap X| \ge k \}$ where $B(x, r) := \{x' \in \mathcal{X} : |x - x'| \le r \}$ and the $k$-NN set of $x \in \mathcal{X}$ be $N_k(x) := B(x, r_k(x)) \cap X$. 
Then for all $x \in \mathcal{X}$, the $k$-NN regression function with respect to the samples is defined as
\begin{align*} 
f_k(x) := \frac{1}{|N_k(x)|} \sum_{i=1}^n y_i \cdot 1\left[ x_i \in N_k(x) \right].
\end{align*}
\end{definition}

Next, we define the following pointwise modulus of continuity, which will be used to express the bias for an arbitrary function in later result.
\begin{definition} [Modulus of continuity]
$u_f(x, r) := \sup_{x' \in B(x, r)} |f(x) - f(x')|$.
\end{definition}


We now state our main result about $k$-NN regression. 
Informally, it says that under the mild assumptions described above, for $k \gtrsim \log n$,
$|f_k(x) - f(x)| \lesssim u_f(x, (k/n)^{1/D}) + \sqrt{(\log n)/k}$ uniformly in $x \in \mathcal{X}$ with high probability.

The first term correponds to the bias term. Using uniform VC-type concentration bounds, it can be shown that the $k$-NN radius
can be uniformly bounded by approximately distance $(k/n)^{1/D}$ and hence no point in the $k$-NN set will be that far. The bias can then be expressed 
in terms of that distance and $u_f$.

The second term corresponds to the variance. The $1/\sqrt{k}$ factor is not surprising since the noise terms are averaged over $k$ observations
and the extra $\sqrt{\log n}$ factor comes from the cost of obtaining a uniform bound.

\begin{definition}
Let $v_D$ be the volume of a $D$-dimensional unit ball.
\end{definition}

\begin{theorem} [$k$-NN Regression Rate] \label{theo:knn}
Suppose that Assumptions~\ref{a1},~\ref{a2}, and~\ref{a3} hold and that
\begin{align*}
2^8 \cdot D \log^2(4/\delta) \cdot \log n \le k \le \frac{1}{2} \cdot \gamma \cdot p_{X,0} \cdot v_D \cdot r_0^D \cdot n.
\end{align*}
Then probability at least $1 - \delta$, the following holds uniformly in $x \in \mathcal{X}$.
\begin{align*}
|f(x) - f_k(x)| &\le u_f\left(x,\left(\frac{2k}{\gamma \cdot p_{X, 0}\cdot v_D\cdot n}\right)^{1/D}\right) \\
&+  2\sigma \sqrt{\frac{D \log n + \log(2/\delta)}{k}}.
\end{align*}
\end{theorem}

Note that the above result is fairly general and makes no smoothness assumptions. In particular, $f$ need not even be continuous. It is also important to point out that $n$ must be sufficiently large in order for there to exist a $k$ that satisfies the conditions. We can then apply this to the class of H\"older continuous functions to obtain the following result.

\begin{corollary}[Rate for $\alpha$-H\"older continuous functions] \label{theo:knnholder}
Let $0 < \alpha \le 1$. Suppose that Assumptions~\ref{a1},~\ref{a2}, and~\ref{a3} hold and
\begin{align*}
2^8 \cdot D \log^2(4/\delta) \cdot \log n \le k \le \frac{1}{2} \cdot \gamma \cdot v_D \cdot p_{X,0}\cdot  r_0^D \cdot n.
\end{align*}
If $f$ is H\"older continuous (i.e. $|f(x) - f(x')| \le C_\alpha |x-x'|^\alpha$), then the following holds:
\begin{align*}
\mathbb{P} \Bigg(\sup_{x \in \mathcal{X}} &|f(x) - f_k(x)| \le C_\alpha \left(\frac{2k}{\gamma\cdot p_{X, 0}\cdot v_D\cdot n}\right)^{\alpha/D} \\
&+ 2\sigma \sqrt{\frac{D \log n + \log(2/\delta)}{k}} \Bigg) \ge 1 - \delta.
\end{align*}
\end{corollary}

\begin{remark}
Taking $k = O(n^{2\alpha/(2\alpha + D)})$ gives us a rate of $$\sup_{x\in\mathcal{X}} |f(x) - f_k(x)|_\infty \lesssim \widetilde{O}(n^{-\alpha/(2\alpha + D)}),$$ which is the minimax optimal rate for estimating a H\"older function, up to logarithmic factors.
\end{remark}

\begin{remark}
It is understood that all our results will also hold under the assumption that the $x_i$'s are fixed and deterministic (e.g. on a grid) as long as there is a sufficient covering of the space.
\end{remark}

\section{Regression On Manifolds}
In this section, we show that if the data has a lower intrinsic dimension, then 
$k$-NN will automatically attain rates as if it were in the lower dimensional space 
and independent of the ambient dimension.

We make the following regularity assumptions which are standard among works 
in manifold learning e.g. \citep{genovese2012minimax, balakrishnan2013cluster}.
\begin{assumption}\label{manifold}
$\mathcal{P}$ is supported on $M$ where:
\begin{itemize}
\itemsep0em 
\item $M$ is a $d$-dimensional smooth compact Riemannian manifold without boundary embedded in compact subset $\mathcal{X} \subseteq \mathbb{R}^D$.
\item The volume of $M$ is bounded above by a constant.
\item $M$ has condition number $1/\tau$, which controls the curvature and prevents self-intersection.
\end{itemize}
Let $p_X$ be the density of $\mathcal{P}$ with respect to the uniform measure on $M$.
\end{assumption}

We now give the manifold analogues of Theorem~\ref{theo:knn} and Corollary~\ref{theo:knnholder}.
\begin{theorem} [$k$-NN Regression Rate] \label{theo:knnmanifold}
Suppose that Assumptions~\ref{a2},~\ref{a3}, and~\ref{manifold} hold and that
\begin{align*}
&k \ge 2^8 \cdot D \log^2(4/\delta) \cdot \log n \\
&k \le \frac{1}{4} \left(\min \left\{ \frac{\tau}{4d}, \frac{1}{\tau} \right\} \right)^d p_{X, 0}\cdot v_d \cdot n.
\end{align*}
Then with probability at least $1 - \delta$, the following holds uniformly in $x \in \mathcal{X}$.
\begin{align*}
|f(x) - f_k(x)| &\le u_f\left(x,\left( \frac{4k}{ v_d \cdot n \cdot p_{X, 0}}\right)^{1/d}\right) \\
&+  2\sigma \sqrt{\frac{D \log n + \log(2/\delta)}{k}}.
\end{align*}
\end{theorem}

Similar to the full dimensional case, we can then apply this to the class of H\"older continuous functions.
\begin{corollary}[Rate for $\alpha$-H\"older continuous functions] \label{theo:knnholdermanifold}
Let $0 < \alpha \le 1$. Suppose that Assumptions~\ref{a2},~\ref{a3}, and~\ref{manifold} hold and
\begin{align*}
&k \ge  2^8 \cdot D \log^2(4/\delta) \cdot \log n \\
&k \le \frac{1}{4} \left(\min \left\{ \frac{\tau}{4d}, \frac{1}{\tau} \right\} \right)^d p_{X, 0}\cdot v_d \cdot n.
\end{align*}
If $f$ is H\"older continuous (i.e. $|f(x) - f(x')| \le C_\alpha |x-x'|^\alpha$), then the following holds
\begin{align*}
\mathbb{P} \Bigg(\sup_{x \in \mathcal{X}} &|f(x) - f_k(x)| \le C_\alpha \left( \frac{4k}{ v_d \cdot n \cdot p_{X, 0}}\right)^{\alpha/d} \\
&+ 2\sigma \sqrt{\frac{D \log n + \log(2/\delta)}{k}} \Bigg) \ge 1 - \delta.
\end{align*}
\end{corollary}

\begin{remark}
Taking $k = O(n^{2\alpha/(2\alpha + d)})$ gives us a rate of $\widetilde{O}(n^{-\alpha/(2\alpha + d)})$, which is more attractive than 
the full dimensional version $\widetilde{O}(n^{-\alpha/(2\alpha + D)})$ when intrinsic dimension $d$ is lower than ambient dimension $D$. We note that the bound contains a constant factor depending on $D$ but the rate at which it decreases as $n$ grows does not.
\end{remark}

\section{Level Set Estimation}

The level set is the region of the input space that have value greater than a fixed threshold.
\begin{definition} [Level-Set]
\begin{align*}
L_f(\lambda) := \{ x \in \mathcal{X} : f(x) \ge \lambda \}.
\end{align*}
\end{definition}

In order to estimate the level-sets, we require the following regularity assumption.
It states that for each maximal connected component of the level-set, the change in the function around the boundary
has a Lipschitz form with smoothness and curvature $\beta > 0$ around some neighborhood of the boundary. 
This notion of regularity at the boundaries of the level-sets is a standard one in density level-set estimation e.g. \citet{tsybakov1997nonparametric,singh2009adaptive}.

\begin{definition} [Level-Set Regularity]
Let $d(x, C) := \inf_{x' \in C} |x-x'|$, $\partial C$ be the boundary of $C$, and $C \oplus r := \{x' : d(x', C) \le r \}$.
A function $f$ satisfies $\beta$-regularity at level $\lambda$ if the following holds.
There exists $r_M , \check{C}, \hat{C} > 0$ such that
for each maximal connected subset $C \subseteq L_f(\lambda)$, we have
\begin{align*}
\check{C} \cdot d(x, \partial C)^\beta \le |\lambda - f(x)|  \le \hat{C} \cdot d(x, \partial C)^\beta,
\end{align*}
for all $x \in \partial C \oplus r_M$.
\end{definition}

\begin{remark}
The upper bound on $|\lambda - f(x)|$ ensures that $f$ is sufficiently smooth so that $k$-NN regression will give us sufficiently accurate estimates near the boundaries. The lower bound on $|\lambda - f(x)|$ ensures that the level-set is salient enough to be detected.
\end{remark}

To recover $L_f(\lambda)$ based on the samples, we use the following estimator, where $X := \{x_1,...,x_n\}$.
\begin{align*}
\widehat{L}(\lambda) := \{ x \in X : f_k(x) \ge \lambda - \epsilon \},
\end{align*}
where $\epsilon := 4\hat{\sigma} \sqrt{\frac{D\log n + \log(2/\delta)}{k}}$ and $\hat{\sigma} := \sqrt{\frac{2}{n} \sum_{i=1}^m y_i^2}$. It will become clear later in the proofs that $\hat{\sigma}$ is meant to be an upper bound on $\sigma$ and thus $\epsilon$ is an upper bound on twice the variance of term of the $k$-NN bound. 

There are three simple but key differences of our estimator when compared to 
 $L_f(\lambda)$. The first is that since we don't have access to the true function $f$, we use the $k$-NN regression estimate $f_k$. Next, instead of taking $x \in \mathcal{X}$, we instead restrict to the samples $X$. This makes our estimator feasible to compute since it will be a subset of the sample points. Finally, we have the $\epsilon$ to bound the uniform deviation of $|f_k - f|$ near the boundary of the level-set (as will be apparent in the proof).
 The main difficulty is choosing $\epsilon$ large enough to bound this uniform deviation, but not too large to overestimate the level-set and finally ensuring that $\epsilon$ can be computed without knowledge of $f$ or any unknown constants (we only need confidence parameter $\delta$ and the dimension, as well as $k$). 
 Thus, our estimator is practical.

We provide consistency result under the Hausdorff metric. We note that this is a strong notion of consistency since it a uniform guarantee on the constituents of our estimator.
\begin{definition}[Hausdorff Distance]
\begin{align*}
d_H(X, Y) = \inf\{\epsilon \ge 0 : X \subseteq Y \oplus \epsilon, Y \subseteq X \oplus \epsilon \}.
\end{align*}
\end{definition}

The next result gives us finite-sample consistency rates for our estimator. 
\begin{theorem}[Level Set Recovery]\label{theo:levelset}
Suppose that Assumptions~\ref{a1},~\ref{a2}, and~\ref{a3} hold.
Let $f$ be continuous and
satisfy $\beta$-regularity at level $\lambda$.
Define $M := \sqrt{\mathbb{E}[y_1^2]}$ where the expectation is taken over $p_X$ and $\xi$, and suppose that $n$ is sufficiently large depending on $\xi$, $f$ and $\delta$. 
If $k$ satisfies
\begin{align*}
k &\ge 8\max\left\{1, \frac{40 M^2}{(2\min\{r_M, r_0\})^{2\beta} \check{C}^2} \right\}\log(4/\delta) D\cdot\log n,\\
k &\le (4\sigma^2/\hat{C})^{2D/(2\beta + D)} \cdot (D\log n + \log(4/\delta))^{\beta/(2\beta + D)} \\ &\hspace{1cm} \cdot  (2\gamma \cdot p_{X, 0} \cdot v_D)^{2\beta/(2\beta+D)}\cdot n^{2\beta/(2\beta + D)},
\end{align*}
then with probability at least $1 - 2\delta$,
\begin{align*}
&d_H(L_f(\lambda), \widehat{L}_f(\lambda)) \\ &\le 2\cdot \left(\frac{24M}{\check{C}}\right)^{1/\beta}\cdot (D\log n\cdot \log(2/\delta))^{1/2\beta} \cdot k^{-1/2\beta}.
\end{align*}
\end{theorem}
\begin{remark}
Although the statement may appear obfuscated, it essentially says that as long as $f$ is a continuous function satisfying $\beta$-regularity at level $\lambda$, then if $k$ lies within the following range:
\begin{align*}
    \log n \lesssim k \lesssim n^{2\beta/(2\beta + D)},
\end{align*}
then with high probability,
\begin{align*}
    d_H(L_f(\lambda), \hat{L}_f(\lambda)) \lesssim k^{-1/(2\beta)}.
\end{align*}
\end{remark}

\begin{remark}
Choosing $k$ at the optimal setting $k \approx n^{2\beta/(2\beta + D)}$, we have $\epsilon = \widetilde{O}(n^{-\beta/(2\beta + D)})$.
Then it follows that we recover the level sets at a Hausdorff rate of $\widetilde{O}(n^{-1/(2\beta + D)})$.
This can be compared to the lower bound $O(n^{-1/(2\beta + D)})$ established by \citet{tsybakov1997nonparametric} for estimating the level sets of an unknown density.
\end{remark}

We can give a similar result when the data lies on a lower dimensional manifold. Interestingly, we can use the exact same estimator as before as if we were operating in the full dimensional space.
\begin{theorem}[Level Set Recovery on Manifolds]\label{theo:levelset-manifold}
Suppose that Assumptions~\ref{a1},~\ref{a2},~\ref{a3}, and~\ref{manifold} hold. Let $f$ be continuous and
satisfy $\beta$-regularity at level $\lambda$.
Define $M := \sqrt{\mathbb{E}[y_1^2]}$ where the expectation is taken over $p_X$ and $\xi$, and suppose that $n$ is sufficiently large depending on $\xi$, $f$, $\tau$, and $\delta$. 
If $k$ satisfies
\begin{align*}
k &\ge 8\max\left\{1, \frac{40 M^2}{(2\min\{r_M, r_0\})^{2\beta} \check{C}^2} \right\}\log(4/\delta) D\cdot\log n,\\
k &\le (4\sigma^2/\hat{C})^{2d/(2\beta + d)} \cdot (D\log n + \log(4/\delta))^{\beta/(2\beta + d)}  \\ &\hspace{1cm} \cdot  ( p_{X, 0} \cdot v_D)^{2\beta/(2\beta+d)}\cdot n^{2\beta/(2\beta + d)},
\end{align*}
then with probability at least $1 - 2\delta$,
\begin{align*}
&d_H(L_f(\lambda), \widehat{L}_f(\lambda)) \\
&\le 2\cdot \left(\frac{24M}{\check{C}}\right)^{1/\beta}\cdot (D\log n\cdot \log(2/\delta))^{1/2\beta} \cdot k^{-1/2\beta}.
\end{align*}
\end{theorem}
\begin{remark}
The main difference from the full-dimensional version is that we need $k$ to satisfy
\begin{align*}
    \log n \lesssim k \lesssim n^{2\beta/(2\beta + d)}.
\end{align*}
Choosing $k$ at the optimal setting $k \approx n^{2\beta/(2\beta + d)}$, we recover the level sets at a rate of $\widetilde{O}(n^{-1/(2\beta + d)})$.
\end{remark}
Remarkably, we obtain the rate as if we were operating on the lower dimensional space. This has not been shown for level-set estimation on manifolds for density functions (which is a different problem). 

The rate for density functions under similar regularity assumptions is
$\widetilde{O}(n^{-1/(2\beta + d\cdot \max\{1, \beta\})})$ \cite{jiang2017density}, which is slower. In other words, we escape the curse of dimensionality with regression level-set estimation but do not escape it for density level-set estimation.

\section{Global Maxima Estimation}
In this section, we give guarantees on estimating the global maxima of $f$.
\begin{definition}
$x_0$ is a maxima of $f$ if $f(x) < f(x_0)$ for all $x \in B(x_0, r) \backslash \{ x_0 \}$ for some $r > 0$.
\end{definition}
We then make the following assumptions, which states that $f$ has a unique maxima, where it has a negative-definite Hessian.
\begin{assumption}\label{maximassumption}
$f$ has a unique maxima $x_0 := \argmax_{x \in \mathcal{X}} f(x)$ and
$f$ has a negative-definite Hessian at $x_0$. 
\end{assumption}

These assumptions lead to the following, which states that $f$ has quadratic smoothness and decay around $x_0$.
\begin{lemma} [\citet{dasgupta2014optimal}] \label{maximaconditions} Let $f$ satisfy Assumption~\ref{maximassumption}.
Then there exists $\hat{C}, \check{C}, r_M, \lambda > 0$ such that the following holds.
\begin{align*}
\check{C} \cdot |x_0 - x|^2 \le f(x_0) - f(x) \le \hat{C} \cdot |x_0 - x|^2
\end{align*}
for all $x \in A_0$ where $A_0$ is a connected component of $\{ x : f(x) \ge \lambda \}$ and 
$A_0$ contains $B(x_0, r_M)$.
\end{lemma}

We utilize the following estimator, which is the maximizer of $f_k$ amongst sample points $X = \{x_1,...,x_n\}$.
\begin{align*}
\widehat{x} := \argmax_{x \in X} f_k(x).
\end{align*}
We next give the result of the accuracy of $\widehat{x}$ in estimating $x_0$.

\begin{theorem} \label{theo:maxima} 
Suppose that $f$ is continuous and that Assumptions~\ref{a1},~\ref{a2},~\ref{a3}, and~\ref{maximassumption} hold.
Let $k$ satisfy
\begin{align*}
& k \ge \frac{2^{10} \cdot D \log^2(4/\delta) \cdot \log n}{\min\{1, \check{C}^2\cdot  r_M^4 / \sigma^2 \}} \\
&  k \le
 \frac{1}{2} \cdot \gamma \cdot p_{X,0}\cdot  v_D \cdot \min\left\{r_0^D, \left(\frac{\check{C}\cdot r_M^2}{32\cdot \hat{C}}\right)^{D/2} \right\} \cdot n.
\end{align*}
Then the following holds with probability at 
least $1 - \delta$.
\begin{align*}
|\hat{x} - x_0|^2 \le \max \bigg\{ &\frac{32 \sigma}{\check{C}} \sqrt{ \frac{D \log n + \log(2/\delta)}{k}}, \\ &\frac{32\hat{C}}{\check{C}}  \bigg(\frac{2k}{\gamma\cdot p_{X, 0} \cdot v_D\cdot n}\bigg)^{2/D} \bigg\}.
\end{align*}
\end{theorem}

\begin{remark}
Taking $k \approx n^{4/(4+D)}$ optimizes the above expression so that $|\hat{x} - x_0| \lesssim \widetilde{O}(n^{-1/(4+D)})$. This can be
compared to the minimax rate for mode estimation $O(n^{-1/(4+D)})$ established by \citet{tsybakov1990recursive}. We stress however that estimating the mode of density function is a different problem.
\end{remark}

\begin{remark}
An analogue for global minima also holds. Moreover, in the manifold setting, we can obtain a rate of $\widetilde{O}(n^{-1/(4+d)})$, which has not been shown for mode estimation in densities.
\end{remark}

\section{Proofs}

\subsection{Proof of Theorem~\ref{theo:knn}}
The follow bounds $r_k(x)$ uniformly in $x \in \mathcal{X}$.
\begin{lemma} \label{rkbound}
The following holds with probability at least $1 - \delta/2$.
If

\begin{align*}
2^8 \cdot D \log^2(4/\delta) \cdot \log n \le k \le \frac{1}{2} \cdot \gamma  \cdot p_{X,0}\cdot v_D \cdot r_0^D \cdot n,
\end{align*}
then
$\sup_{x \in \mathcal{X}} r_k(x) \le \left( \frac{2k}{\gamma \cdot v_D \cdot n \cdot p_{X, 0}}\right)^{1/D}$.
\end{lemma}
\begin{proof}
Let $r = \left( \frac{2k}{\gamma \cdot v_D \cdot n \cdot p_{X, 0}}\right)^{1/D}$. We have
$\mathcal{P}(B(x, r)) \ge \gamma \inf_{x' \in B(x, r) \cap \mathcal{X}} p_X(x') \cdot v_D r^D \ge \gamma p_{X, 0} v_D r^D = \frac{2k}{n}$.
By Lemma 7 of \cite{chaudhuri2010rates} and the condition on $k$, it follows that with probability $1 - \delta/2$, uniformly in $x \in \mathcal{X}$, 
$\mathcal{P}_n(B(x, r)) \ge \frac{k}{n}$. Hence, $r_k(x) < r$ and the result follows immediately.
\end{proof}

The next result bounds the number of distinct $k$-NN sets over $\mathcal{X}$.
\begin{lemma} \label{knncount}
Let $M$ be the number of distinct $k$-NN sets over $\mathcal{X}$, that is, $M := |\{ N_k(x) : x \in \mathcal{X} \}|$. 
Then $M \le D\cdot n^D$.
\end{lemma}

\begin{proof}
First, let $\mathcal{A}$ be the partitioning of $\mathcal{X}$ induced by the $\binom{n}{2}$ hyperplanes defined as the perpendicular bisectors
of each pair of points $x_i$, $x_j$ for $i \neq j$. Let us denote this set of hyperplanes as $\mathcal{H}$.
We have that if $x, x'$ are in the same partition of $\mathcal{A}$, then $N_k(x) = N_k(x')$. If not, then any path from $x$ to $x'$ must cross some perpendicular bisector in $N_k(x') - N_k(x)$, which would be a contradiction.
Thus, $M \le |\mathcal{A}|$.

Now we will bound $|\mathcal{A}|$. Since $\mathcal{H}$ is finite, choose vectors $e_1,...,e_D$ such that they form an orthogonal basis of $\mathbb{R}^D$ and none of these vectors are perpendicular to any $H \in \mathcal{H}$. 
Let $e_1,...,e_D$ induce hyperplanes $H_1,...,H_D$, respectively (i.e. $H_i$ being the orthogonal complement of $e_i$).
Without loss of generality, orient the space such that $e_1$ is the vertical direction (i.e. so that we can use descriptions such as 'above' and 'below').
For each region in $\mathcal{A}$ that is bounded below, associate such a region to its lowest point. Then it follows that there are at most $\binom{n}{D}$ of these regions since they are the intersection of $D$ hyperplanes.

We next count the regions unbounded below. Place $H_1$ below the lowest point corresponding the regions in $\mathcal{A}$ that were bounded below.
Then we have that the regions unbounded below are $\{ A \in \mathcal{A} : A \cap H_1 \neq \emptyset\}$. It thus remains now to count
$\mathcal{A}_{1} := \{ A\cap H_1 : A \in \mathcal{A}, A \cap H_1 \neq \emptyset\}$.

We now orient the space so that $e_2$ corresponds to the vertical direction. Then we can repeat the same procedure and for each region in $\mathcal{A}_1$that is bounded below with the lowest point. There are at most $\binom{n}{D - 1}$ since they are an intersection of $D-1$ hyperplanes in $\mathcal{H}$ along with $H_1$, and then placing $e_2$ sufficiently low, the remaining regions correspond to
$\mathcal{A}_{2} := \{ A\cap H_1 \cap H_2 : A \in \mathcal{A}, A \cap H_1 \cap H_2 \neq \emptyset\}$.

Continuing this process, it follows that when we orient $e_i$ to be the vertical direction,
in order to count $\mathcal{A}_{i} := \{ A\cap H_1 \cap \cdots \cap H_i : A \in \mathcal{A}, A\cap H_1 \cap \cdots \cap H_i \neq \emptyset\}$, the number of regions in $\mathcal{A}_i$ bounded below is at most $\binom{n}{D-i}$ and the remaining ones are correspond to $\mathcal{A}_{i+1}$. 

It thus follows that
$|\mathcal{A}| \le \sum_{j=0}^D \binom{n}{j} \le D \cdot n^D$,
as desired.
\end{proof}

\begin{proof} [Proof of Theorem~\ref{theo:knn}]
We have
\begin{align*}
&|f_k(x) - f(x)|
\le  \left|\frac{1}{|N_k(x)|}\sum_{i=1}^n (f(x_i)- f(x)) \cdot 1\left[ x_i \in N_k(x) \right] \right|  \\ &\hspace{1cm} + \left|\frac{1}{|N_k(x)|}\sum_{i=1}^n \xi_{x_i} \cdot 1\left[ x_i \in N_k(x) \right] \right|\\
&\le  u_f(x, r_k(x)) + \left|\frac{1}{N_k(x)}\sum_{i=1}^n \xi_{x_i} \cdot 1\left[ x_i \in N_k(x) \right] \right|.
\end{align*}
The first term can be viewed as the bias term and the second can be viewed as variance term.

By Lemma~\ref{rkbound}, we can bound the first term as follows with probability at least $1 - \delta/2$ uniformly
in $x \in \mathcal{X}$: $u_f(x, r_k(x)) \le u_f\left(x,\left(\frac{2k}{\gamma \cdot p_{X, 0} \cdot v_D\cdot n}\right)^{1/D}\right)$.
For the variance term, we have by Hoeffding's inequality that if 
$A_x := \left|\frac{1}{k}\sum_{i=1}^n \xi_{x_i} \cdot 1\left[ x_i \in N_k(x) \right] \right|$
then
$\mathbb{P}\left(A_x > \frac{\sqrt{2} \sigma \cdot t}{\sqrt{k}}\right) \le \exp \left( - t^2 \right)$.

Taking $t = \sqrt{D\log n + \log(2D/\delta)}$, then we have 
$\mathbb{P}\left(A_x > \frac{\sqrt{2} \sigma \cdot t}{\sqrt{k}}\right) \le \delta/ (2 D \cdot n^D)$.

By Lemma~\ref{knncount} and union bound, it follows that 
$\mathbb{P}\left(\sup_{x \in \mathcal{X}} A_x > \frac{\sqrt{2} \sigma \cdot t}{\sqrt{k}}\right) \le \delta/2$.
Hence, we have with probability at least $1 - \delta$,
\begin{align*}
|f(x) - f_k(x)| &\le u_f\left(x,\left(\frac{2k}{\gamma \cdot p_{X, 0}\cdot v_d\cdot n}\right)^{1/D}\right) \\
&+ 2\sigma \sqrt{\frac{D \log n + \log(2/\delta)}{k}}.
\end{align*}
uniformly in $x \in \mathcal{X}$.
\end{proof}

It is easy to see that a simple modification to the proof of Theorem~\ref{theo:knn} will yield the following.
\begin{corollary} [$k$-NN Regression Upper and Lower Bounds] \label{corr:knnbounds}
Let
\begin{align*}
\hat{u}_f(x, r) &:= \sup_{x' \in B(x, r)} f(x') - f(x) \\
\check{u}_f(x, r) &:= \sup_{x' \in B(x, r)} f(x) - f(x')\\
\varepsilon_{\text{var}} &:= 2\sigma \sqrt{ \frac{D \log n + \log(2/\delta)}{k}} \\
\varepsilon_k &:= \left(\frac{2k}{\gamma p_{X, 0} v_D\cdot n}\right)^{1/D}.
\end{align*}
Suppose that Assumptions~\ref{a1},~\ref{a2}, and~\ref{a3} hold and that
\begin{align*}
k \ge 2^8 \cdot D \log^2(4/\delta) \cdot \log n.
\end{align*}
Then probability at least $1 - \delta$, the following holds uniformly in $x \in \mathcal{X}$.
\begin{align*}
f_k(x) &\le f(x) + \hat{u}_f(x, \varepsilon_k) +  \varepsilon_{\text{var}}\\
f_k(x) &\ge f(x) - \check{u}_f (x, \varepsilon_k) -  \varepsilon_{\text{var}}.
\end{align*}
\end{corollary}

\subsection{Proof of Theorem~\ref{theo:knnmanifold}}
We need the following guarantee on the volume of the intersection of a Euclidean ball and $M$; this is required to get a handle on the true mass of the ball under
$\mathcal{P}$ in later arguments. The proof can be found in \cite{jiang2017density}.

\begin{lemma} [Ball Volume] \label{ballvolume}
If $0 < r < \min\{\tau/(4d), 1/\tau\}$, and $x \in M$ then
\begin{align*}
1 - \tau^2 r^2 &\le \frac{\text{vol}_{d} (B(x, r) \cap M)}{v_{d} r^{d}} \le 1 + 4d\cdot r/\tau,
\end{align*}
where $\text{vol}_{d}$ is the volume w.r.t. the uniform measure on $M$. 
\end{lemma}

The next is the manifold analogue of Lemma~\ref{rkbound}.
\begin{lemma} \label{rkboundmanifold}
Suppose that Assumptions~\ref{a2},~\ref{a3}, and~\ref{manifold} hold. 
The following holds with probability at least $1 - \delta/2$.
If 
\begin{align*}
 2^8 \cdot D \log^2(4/\delta) \cdot \log n \le k \le \frac{1}{4} \left(\min \left\{ \frac{\tau}{4d}, \frac{1}{\tau} \right\} \right)^d p_{X, 0}\cdot v_d \cdot n.
\end{align*}
then for all $x \in M$,
$r_k(x) \le \left( \frac{4k}{ v_d \cdot n \cdot p_{X, 0}}\right)^{1/d}$.
\end{lemma}
\begin{proof}
Let $r =\left( \frac{4k}{ v_d \cdot n \cdot p_{X, 0}}\right)^{1/d}$. We have
\begin{align*}
\mathcal{P}(B(x, r)) &\ge \inf_{x' \in B(x, r) \cap M} p_X(x') \cdot \text{Vol}_d (B(x, r) \cap M) \\
&\ge  p_{X, 0} \cdot (1 - \tau^2 r^2)\cdot v_d r^d \ge \frac{1}{2} p_{X, 0} v_d r^d \ge \frac{2k}{n}.
\end{align*}
By Lemma 7 of \cite{chaudhuri2010rates} and the condition on $k$, it follows that with probability $1 - \delta/2$, uniformly in $x \in \mathcal{X}$, 
$\mathcal{P}_n(B(x, r)) \ge \frac{k}{n}$. Hence, $r_k(x) < r$ and the result follows immediately.
\end{proof}

Theorem~\ref{theo:knnmanifold} now follows by replacing the usage of Lemma~\ref{rkbound} with Lemma~\ref{rkboundmanifold}. We also note that an analogous result to Corollary~\ref{corr:knnbounds} can also be established.

It is easy to see that a simple modification to the proof of Theorem~\ref{theo:knn} will yield the following.
\begin{corollary} [$k$-NN Regression Upper and Lower Bounds] \label{corr:knnbounds}
Let
\begin{align*}
\hat{u}_f(x, r) &:= \sup_{x' \in B(x, r)} f(x') - f(x) \\
\check{u}_f(x, r) &:= \sup_{x' \in B(x, r)} f(x) - f(x')\\
\varepsilon_{\text{var}} &:= 2\sigma \sqrt{ \frac{D \log n + \log(2/\delta)}{k}} \\
\varepsilon_k &:= \left(\frac{2k}{\gamma p_{X, 0} v_D\cdot n}\right)^{1/D}.
\end{align*}
Suppose that Assumptions~\ref{a1},~\ref{a2}, and~\ref{a3} hold and that
\begin{align*}
k \ge 2^8 \cdot D \log^2(4/\delta) \cdot \log n.
\end{align*}
Then probability at least $1 - \delta$, the following holds uniformly in $x \in \mathcal{X}$.
\begin{align*}
f_k(x) &\le f(x) + \hat{u}_f(x, \varepsilon_k) +  \varepsilon_{\text{var}}\\
f_k(x) &\ge f(x) - \check{u}_f (x, \varepsilon_k) -  \varepsilon_{\text{var}}.
\end{align*}
\end{corollary}


\subsection{Proofs of Theorem~\ref{theo:levelset} and~\ref{theo:levelset-manifold}}
\begin{proof}[Proof of Theorem~\ref{theo:levelset}]
We have that $E[\hat{\sigma}^2] = 2M^2 \ge 2\text{Var}(\xi^2) =2\sigma^2$. Thus, when $n$ is sufficiently large depending on $\xi$, $f$, and $\delta$, we have by Bernstein-type concentration inequalities that with probability at least $1 - \delta$, $2\sigma \le \hat{\sigma} \le \sqrt{5}M$. 

Let $\tilde{r} := 2(2\epsilon/\check{C})^{1/\beta}$ and let us use the notation introduced in Corollary~\ref{corr:knnbounds}. It suffices to show that (1) 
$\widehat{L}_f(\lambda) \subseteq L_f(\lambda) \oplus \tilde{r}$ and 
(2) $L_f(\lambda) \subseteq \widehat{L}_f(\lambda) \oplus \tilde{r}$.
We begin with (1). We have
\begin{align*}
    \sup_{x \in \mathcal{X} \backslash (L_f(\lambda) \oplus \tilde{r})} f_k(x)
&\le \sup_{x \in \mathcal{X} \backslash (L_f(\lambda) \oplus \tilde{r})} (f(x)
+ \hat{u}_f(x, \epsilon_k)) + \varepsilon_{\text{var}} \\
&\le \sup_{x \in \mathcal{X} \backslash (L_f(\lambda) \oplus \tilde{r})} \sup_{x' \in B(x, \epsilon_k)} f(x') + \varepsilon_{\text{var}} \\
&=  \sup_{x \in \mathcal{X} \backslash (L_f(\lambda) \oplus (\tilde{r} - \epsilon_k))} f(x) + \varepsilon_{\text{var}} \\
&\le \lambda - \check{C} (\tilde{r} - \epsilon_k)^\beta +  \varepsilon_{\text{var}} \le \lambda - \epsilon,
\end{align*}
where the first inequality holds by Corollary~\ref{corr:knnbounds},
the second-to-last inequality holds by $\beta$-regularity and that $\tilde{r} < r_M$, and the last inequality holds by the conditions on $k$ (which in particular imply $\epsilon \ge 2\varepsilon_{\text{var}}$ and $\epsilon_k < (2\epsilon/\check{C})^{1/\beta}$). 
Thus, if $x \not\in L_f(\lambda) \oplus \tilde{r}$, then $f_k(x) < \lambda - \epsilon$. Therefore, $\widehat{L}_f(\lambda) \subseteq L_f(\lambda) \oplus \tilde{r}$, which establishes (1). 

We now show (2). Let $\bar{r} = \epsilon_k$. Since $\bar{r} < \tilde{r}$, it suffices to show that $L_f(\lambda) \subseteq \widehat{L}_f(\lambda) \oplus \bar{r}$. For any $x \in L_f(\lambda)$, we have
\begin{align*}
    \mathcal{P}(B(x, \bar{r})) 
    \ge \frac{2k}{n} \ge \frac{16 \log(4/\delta) D \log n}{n},
\end{align*}
where the last inequality holds by the conditions on $k$.
Hence, by Lemma 7 of \cite{chaudhuri2010rates}, we have $\mathcal{P}_n(B(x, \bar{r})) > 0$.
Thus, for any $x \in L_f(\lambda)$, there exists a sample point in $B(x, \bar{r})$. Furthermore, we have
\begin{align*}
    \inf_{x' \in B(x, \bar{r})} f_k(x') &\ge  \inf_{x' \in B(x, \bar{r})} f(x) - \check{u}_f(x, \epsilon_k) - \varepsilon_{\text{var}}\\
    &\ge \inf_{x' \in B(x, \bar{r})} \inf_{x'' \in B(x', \epsilon_k)} f(x'') - \varepsilon_{\text{var}}\\
    &= \inf_{x' \in B(x, \bar{r} + \epsilon_k)} f(x') - \varepsilon_{\text{var}}\\
    &\ge \lambda - \hat{C}(\bar{r} + \epsilon_k)^{\beta} - \varepsilon_{\text{var}} \ge \lambda - \epsilon.
\end{align*}
where the first inequality holds by Corollary~\ref{corr:knnbounds}, the second last inequality holds by $\beta$-regularity, and the final inequality holds by the conditions on $k$. 

Thus, for any $x \in L_f(\lambda)$, not only does there exists a sample point in $B(x, \bar{r})$, but any such sample point will have $f_k$ value at least $\lambda - \epsilon$ and thus is in $\widehat{L}_f(\lambda)$. Therefore, $L_f(\lambda) \subseteq \widehat{L}_f(\lambda) \oplus \bar{r}$, 
as desired.
\end{proof}

\begin{proof}[Proof of Theorem~\ref{theo:levelset-manifold}]
The proof is the same as that of Theorem~\ref{theo:levelset} but with the full-dimensional $k$-NN regression bounds replaced by the manifold versions, and is omitted here.
\end{proof}

\subsection{Proof of Theorem~\ref{theo:maxima}}

\begin{proof} [Proof of Theorem~\ref{theo:maxima}]
Define the following.
\begin{align*}
\varepsilon_{\text{var}} &:= 2\sigma \sqrt{ \frac{D \log n + \log(2/\delta)}{k}},
\varepsilon_k := \left(\frac{2k}{\gamma \cdot p_{X, 0} v_D\cdot n}\right)^{1/D}\\
\tilde{r}^2 &:= \max \{ 16 \varepsilon_{\text{var}} / \hat{C}, (2\epsilon_k/c)^2 \},
\end{align*}
where $c^2 = \check{C}/8\hat{C}$. The goal is now to show $|x - x_0| \le \tilde{r}$.
The proof now mirrors that of Theorem 1 of \citet{dasgupta2014optimal}.
It suffices to show that 
\begin{align*}
\sup_{x \in \mathcal{X} \backslash B(x_0, \tilde{r})} f_k(x) < \inf_{x \in B(x_0, r_n)} f_k(x),
\end{align*}
where $r_n = d(x_0, X)$.
We have by Corollary~\ref{corr:knnbounds}:
\begin{align*}
\sup_{x \in \mathcal{X} \backslash B(x_0, \tilde{r})} f_k(x) 
&\le \sup_{x \in \mathcal{X} \backslash B(x_0, \tilde{r})} f(x) + \hat{u}_f(x, \varepsilon_k) + \varepsilon_{\text{var}}\\
&\le \sup_{x \in \mathcal{X} \backslash B(x_0, \tilde{r})} f(x) + \hat{u}_f(x, \tilde{r}/2) + \varepsilon_{\text{var}}\\
&\le \sup_{x \in \mathcal{X} \backslash B(x_0, \tilde{r}/2)} f(x) + \varepsilon_{\text{var}} \\
&\le f(x_0) - \check{C}(\tilde{r}/2)^2 + \varepsilon_{\text{var}}.
\end{align*}
On the other hand, 
\begin{align*}
\inf_{x \in B(x_0, r_n)} f_k(x) 
&\ge \inf_{x \in B(x_0, r_n)} f(x) - \check{u}_f(x, \varepsilon_k) - \varepsilon_{\text{var}} \\
&\ge \inf_{x \in B(x_0, c \tilde{r}/2)} f(x) - \check{u}_f(x, c\tilde{r}/2) - \varepsilon_{\text{var}}\\
&\ge \inf_{x \in B(x_0, c\tilde{r})} f(x) - \varepsilon_{\text{var}} \\
&\ge f(x_0) - \hat{C} (c\tilde{r})^2 - \varepsilon_{\text{var}}.
\end{align*}
The result now follows from our choice of $\tilde{r}$.
\end{proof}
{\bf Conclusion: }
We provided finite-sample sup-norm bounds for $k$-NN regression under standard nonparametric assumptions for both the full-dimensional and manifold setting.
We then applied our results to level-set and global maxima estimation.

\newpage
{
\bibliography{paper}
\bibliographystyle{plainnat}
}

\end{document}